\documentclass[sigconf]{acmart}

\AtBeginDocument{%
  }
    
\copyrightyear{2025} 
\acmYear{2025} 
\setcopyright{cc}
 \setcctype{by-nc}
\acmConference[KDD '25]{Proceedings of the 31st ACM SIGKDD Conference on Knowledge Discovery and Data Mining V.2}{August 3--7, 2025}{Toronto, ON, Canada} \acmBooktitle{Proceedings of the 31st ACM SIGKDD Conference on Knowledge Discovery and Data Mining V.2 (KDD '25), August 3--7, 2025, Toronto, ON, Canada} 
 \acmDOI{10.1145/3711896.3737162} 
\acmISBN{979-8-4007-1454-2/2025/08}

\usepackage{url}
\usepackage{enumitem}
\usepackage{graphicx}
\usepackage{multirow}
\usepackage{booktabs}
\usepackage{wrapfig}
\usepackage{amsthm,amsmath}
\usepackage{diagbox}
\usepackage{float}
\usepackage{algorithm}
\usepackage{algorithmic}
\newtheorem{proposition}{Proposition}
\usepackage{listings}
\usepackage{xcolor}      
\usepackage{cleveref}    
\usepackage{hyperref}  
\usepackage{fancyhdr}
\usepackage{booktabs}       
\usepackage{threeparttable} 
\usepackage{caption}
\usepackage{pifont}

\begin{document}

\title{Towards Efficient Few-shot Graph Neural Architecture Search via Partitioning Gradient Contribution}

\author{Wenhao Song}
\authornote{Equal Contribution.}
\affiliation{%
  \institution{College of Software, Jilin University}
  \city{Changchun}
  \state{Jilin}
  \country{China}
}
\email{songwh23@mails.jlu.edu.cn}

\author{Xuan Wu}
\authornotemark[1]
\affiliation{%
  \institution{College of Computer Science and Technology, Jilin University}
  \city{Changchun}
  \state{Jilin}
  \country{China}
}
\email{wuuu22@mails.jlu.edu.cn}

\author{Bo Yang}
\affiliation{%
  \institution{Key Laboratory of Symbolic Computation and Knowledge Engineering of Ministry of Education, Jilin University}
  \city{Changchun}
  \state{Jilin}
  \country{China}
}
\email{ybo@jlu.edu.cn}

\author{You Zhou}
\authornotemark[2]
\affiliation{%
  \institution{College of Software, Jilin University}
  \city{Changchun}
  \state{Jilin}
  \country{China}
}
\email{zyou@jlu.edu.cn}

\author{Yubin Xiao}
\affiliation{%
  \institution{College of Computer Science and Technology, Jilin University}
  \city{Changchun}
  \state{Jilin}
  \country{China}
}
\email{xiaoyb21@mails.jlu.edu.cn}

\author{Yanchun Liang}
\affiliation{%
  \institution{School of Computer Science, Zhuhai College of Science and Technology}
  \city{Zhuhai}
  \state{Guangdong}
  \country{China}
}
\email{ycliang@jlu.edu.cn}

\author{Hongwei Ge}
\affiliation{%
  \institution{College of Computer Science and Technology, Dalian University of Technology}
  \city{Dalian}
  \state{Liaoning}
  \country{China}
}
\email{hwge@dlut.edu.cn}

\author{Heow Pueh Lee}
\affiliation{%
  \institution{Department of Mechanical Engineering, National University of Singapore}
  \country{Singapore}
}
\email{mpeleehp@nus.edu.sg}

\author{Chunguo Wu}
\authornote{Corresponding Authors.}
\affiliation{%
  \institution{College of Computer Science and Technology, Jilin University}
  \city{Changchun}
  \state{Jilin}
  \country{China}
}
\email{wucg@jlu.edu.cn}
\renewcommand{\shortauthors}{Wenhao Song et al.}

\begin{abstract}
\noindent To address the weight coupling problem, certain studies introduced few-shot Neural Architecture Search (NAS) methods, which partition the supernet into multiple sub-supernets. However, these methods often suffer from computational inefficiency and tend to provide suboptimal partitioning schemes. To address this problem more effectively, we analyze the weight coupling problem from a novel perspective, which primarily stems from distinct modules in succeeding layers imposing conflicting gradient directions on the preceding layer modules. Based on this perspective, we propose the Gradient Contribution (GC) method that efficiently computes the cosine similarity of gradient directions among modules by decomposing the Vector-Jacobian Product during supernet backpropagation. Subsequently, the modules with conflicting gradient directions are allocated to distinct sub-supernets while similar ones are grouped together. To assess the advantages of GC and address the limitations of existing Graph Neural Architecture Search methods, which are limited to searching a single type of Graph Neural Networks (Message Passing Neural Networks (MPNNs) or Graph Transformers (GTs)), we propose the Unified Graph Neural Architecture Search (UGAS) framework, which explores optimal combinations of MPNNs and GTs. The experimental results demonstrate that GC achieves state-of-the-art (SOTA) performance in supernet partitioning quality and time efficiency. In addition, the architectures searched by UGAS+GC outperform both the manually designed GNNs and those obtained by existing NAS methods. Finally, ablation studies further demonstrate the effectiveness of all proposed methods.
\end{abstract}

\ccsdesc[500]{Computing methodologies~Artificial intelligence}
\ccsdesc[500]{Computing methodologies~Search methodologies}

\keywords{Neural Architecture Search; Graph Neural Network; Weight Coupling Problem; Supernet Partitioning Method}
\maketitle

\newcommand\kddavailabilityurl{https://doi.org/10.5281/zenodo.15520482}

\ifdefempty{\kddavailabilityurl}{}{
\begingroup\small\noindent\raggedright\textbf{KDD Availability Link:}\\

The source code of this paper has been made publicly available at Source code: \url{\kddavailabilityurl}.
\endgroup
}

\section{Introduction}
\noindent In recent years, Neural Architecture Search (NAS) methods \cite{chen_llm, QIU2023110671, zhou2024de, wu2024} have achieved remarkable performance across various tasks, e.g., discovering novel Graph Neural Network (GNN) architectures. However, traditional NAS methods are often computationally expensive, because they require training and evaluating multiple candidate architectures from scratch during the whole search process \cite{Zoph2017,li_survey,wu_incorporating_2023}.

To reduce the required computing resources, certain studies have proposed one-shot NAS methods, e.g., OS \cite{Pham2018} and SPOS \cite{Guo2020}. Specifically, these methods conceptualize a single supernet as a directed acyclic graph that encapsulates the entire search space, with subnets (i.e., candidate architectures) defined as specific subgraphs sampled from the supernet (see Figure~\ref{Supernet}(A)). During the search process, only the supernet is trained to obtain weights. Subsequently, these trained weights are directly reused to evaluate the performance of each subnet, significantly reducing computational costs compared to traditional NAS methods.

\begin{figure}[!t]
\centering
\includegraphics[width=1\columnwidth]{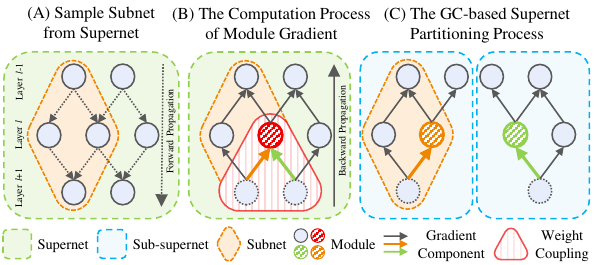}
\caption{(A) illustrates the process of sampling a subnet from the supernet. (B) highlights the weight coupling issue within the shadowed module, which arises from modules in succeeding layers providing conflicting gradient directions to this module. (C) illustrates the proposed Gradient Contribution (GC)-based supernet partitioning method, which allocates the modules with conflicting gradient directions into distinct sub-supernets. Therefore, the weights inherited by the subnet are more aligned with those obtained when training this subnet from scratch, significantly alleviating the weight coupling problem.} 
\vspace{-0.4cm}
\label{Supernet}
\end{figure}

However, because the supernet couples multiple subnets through weight sharing, directly using its weights to evaluate the subnets may result in performance deviations compared to training each subnet independently from scratch \cite{Guo2020, Pourchot2020, Zhao2021}. To address this problem, certain few-shot NAS methods have been proposed. For example, Zhao et al. \cite{Zhao2021} proposed the FS method, which selects specific layers and then exhaustively partitions the supernet into multiple sub-supernets based on the number of candidate modules in the selected layers. Subsequently, each subnet directly inherits the weights from the corresponding sub-supernet. Instead of exhaustively partitioning, Hu et al. \cite{Hu2022} introduced a partitioning rule based on the similarity of the gradient matching score, where modules with higher similarity gradients are grouped into the same sub-supernet. However, computing the gradient of each module requires retraining the temporary supernet (pruning all other modules in the corresponding layer), leading to substantial computing overhead.

Unlike these few-shot NAS methods, we introduce a novel perspective to analyze the weight coupling issue, which arises from the conflicting gradient directions contributed by multiple modules in the $l+1$th layer to the modules in the $l$th layer, as shown in Figure~\ref{Supernet}(B). Therefore, in this work, we investigate the following research question: 

\noindent \textit{Can modules with conflicting gradient directions be allocated into distinct sub-supernets to effectively address the weight coupling issue without introducing additional architectural training overhead?}

\noindent To address this research question, we propose a novel partitioning method called Gradient Contribution (\textbf{GC}). Specifically, GC first decomposes the Vector-Jacobian Product (VJP) \cite{vjp} of the \mbox{$l+1$th} layer modules to compute the gradient contributions of the $l+1$th layer modules to the $l$th layer modules during backpropagation, and then computes the cosine similarity of these gradient contributions. To allocate modules with conflicting gradient directions into distinct sub-supernets, GC formulates the task as a Minimum Cut Problem (MCP), treating the $l+1$th layer modules as nodes and the cosine similarity values as the edge weights between nodes. An MCP solver, e.g., the Stoer-Wagner algorithm \cite{stoer1997simple}, is subsequently employed to partition these $l+1$th layer modules into two disjoint sets. By traversing all layers except the first, GC identifies the top-$k$ layers with the smallest cut weights, ultimately partitioning the supernet into the number of $2^k$ sub-supernets. This approach ensures that the modules with higher similarity in gradient contributions are grouped into the same sub-supernet, while the modules with conflicting gradient directions are assigned to distinct sub-supernets.

To assess the effectiveness of GC, we focus on the task of discovering novel GNN architectures and propose the Unified Graph Neural Architecture Search (\textbf{UGAS}) framework. In contrast to prior Graph Neural Architecture Search (GNAS) studies \cite{PAS,MTGC,AutoGT}, which limit their search to either Message Passing Neural Networks (MPNNs) or Graph Transformers (GTs) due to the significant structural differences between MPNNs and GTs, UGAS concurrently explores both MPNNs and GTs, striving for superior performance. Specifically, UGAS incorporates four MPNNs and two GTs as optional modules within the backbone block of candidate subnets for all datasets. In addition, we design a tailored mechanism to fuse the features captured by different kinds of GNN modules. \textbf{To the best of our knowledge, this work is the first to introduce a unified GNAS framework capable of searching for the optimal combination of MPNNs and GTs.} In addition, it is worth noting that \textbf{the proposed UGAS is a generic framework compatible with different NAS search paradigms}, e.g., Genetic Algorithm (GA)-based and differentiable-based paradigms (see Section~\ref{sec4.2} and Appendix~\ref{appendix:darts} for more details).

To verify the effectiveness of the proposed GC method and UGAS framework, we conduct extensive experiments on ten widely adopted GNN datasets \cite{Benchmark_GNN, LRGB, hu2020ogb}. The experimental results demonstrate that the architectures searched by both differentiable-based and GA-based implementations of UGAS outperform twelve manually designed GNNs and other GNAS-searched architectures across all datasets. In addition, GC outperforms state-of-the-art (SOTA) few-shot NAS methods in terms of supernet partitioning quality and time efficiency. Note that GC requires only 21\% of the time needed by GM \cite{Hu2022} for partitioning the supernet, while achieving a superior partitioning scheme. Finally, ablation studies demonstrate the effectiveness of integrating MPNNs and GTs, as well as parameter settings.

The key contributions of this work are as follows.

\textbf{I)} We propose a novel perspective to analyze the weight coupling problem, which arises from the conflicting gradients contributed by multiple modules in the succeeding layer to those in the preceding layer.

\textbf{II)} We propose the GC method, which decomposes the VJP to efficiently compute the similarity of gradient directions among modules within the same layer, thereby partitioning the supernet.  

\textbf{III)} We propose the UGAS framework, which incorporates diverse MPNNs and GTs as optional modules of the GNN backbone block, and effectively fuses their output features. 

\textbf{IV)} The experimental results demonstrate that our proposed GC achieves SOTA performance in terms of supernet partition quality and time efficiency. In addition, the architectures searched by UGAS outperform both the manually designed and other GNAS-searched ones across diverse GNN datasets.

\vspace{-0.2cm}
\section{Related Work}
\noindent This section first presents recent NAS methods, followed by an introduction of the manually designed GNNs.

\subsection{Neural Architecture Search (NAS)}
\label{sec2.2}
\noindent NAS aims to automatically design well-performing network architectures within a predefined search space. However, most NAS methods require substantial computing resources to train and evaluate all candidate architectures \cite{Zoph2018, Pham2018, Real2019}. In contrast to these approaches, Liu et al. \cite{Liu2019} incorporated the concept of one-shot NAS and proposed OS method, which trains a single supernet encompassing all candidate architectures (subnets), with each subnet directly inheriting the weights from the supernet. While efficient, one-shot NAS methods suffer from weight coupling among subnets, resulting in unreliable weights.

To address this issue, recent NAS studies \cite{Zhao2021, Hu2022} introduced the concept of few-shot NAS and partitioned the supernet into multiple sub-supernets, allowing subnets to inherit weights from their corresponding sub-supernets instead of directly from the supernet. For example, Oh et al. \cite{SBS} proposed a zero-cost method to partition supernet. Specifically, they quantified the number of nonlinear functions in each module and assigned them to different sub-supernets, aiming to keep the number of nonlinear functions across sub-supernets as balanced as possible. However, these prior few-shot NAS methods fail to account for the underlying cause of the weight coupling problem, which primarily arises from the conflicting gradient directions of modules within the same layer during backpropagation. This critical oversight leads to suboptimal supernet partitioning schemes. In contrast to these methods, we mitigate the weight coupling problem by decomposing the gradient directions during backpropagation and partitioning the supernet based on the cosine similarity of gradient directions across modules.

\subsection{Graph Neural Network (GNN)}
 \noindent GNNs can be broadly divided into two categories, i.e., MPNNs and GTs. MPNNs (e.g., GCN \cite{GCN}, and GIN \cite{GIN}) consist of two key stages, i.e., message passing and readout \cite{Gilmer2017}, which are highly effective in processing graph-structured data by aggregating information from neighboring nodes \cite{Gilmer2017}. However, they face challenges such as over-smoothing  (i.e., the difficulty in distinguishing node features) \cite{oono2019} and over-squashing (i.e., the difficulty in transmitting information over long distances) \cite{Alon2021, Topping2022}. As another category of GNNs, GTs (e.g., SAN \cite{SAN}, GRIT \cite{GRIT}) leverage global attention mechanism to capture both local and global information. While global attention addresses the limitations of MPNNs, it introduces a computational complexity that grows quadratically with the number of nodes, posing significant challenges when applied to large-scale graphs.

Recent studies have proposed GNN frameworks that integrate MPNNs and GTs \cite{GraphTrans, GraphGPS, TIGT}. These methods preserve the sensitivity of MPNNs to local structures while enhancing the ability to process global information through the incorporation of GTs. However, these methods require substantial manual effort to design the combination of MPNNs and GTs, which may result in suboptimal combinations. To better address this problem, we propose UGAS, a GNAS framework designed to identify the optimal combination of MPNNs and GTs. Notably, UGAS is a general framework that can be seamlessly integrated into various NAS search paradigms. This work implements UGAS using Evolutionary Computation and Differentiable Gradient methods, with the Reinforcement Learning-based version left for future work.

\vspace{-0.2cm}
\section{Preliminary}
\noindent This section introduces the concepts of the Vector-Jacobian Product and the Minimum Cut Problem.
\vspace{-0.2cm}
\subsection{Vector-Jacobian Product (VJP)}
\noindent The Vector-Jacobian Product (VJP) \cite{vjp, autodiff} is a fundamental operation in the backpropagation algorithm, crucial for computing gradients during network training. Formally, let \(f:\mathbb{R}^{d_1} \rightarrow \mathbb{R}^{d_2}\) be a neural network mapping a \(d_1\)-dimensional input vector \(x\) to a \(d_2\)-dimensional output vector \(y\). The VJP operation is defined as the product of two mathematical entities: the transpose of the Jacobian matrix of \(f\) evaluated at \(x_0\), and the gradient of the loss function \(\mathcal{L}\) with respect to the model's output \(y\). This operation is expressed as:
\begin{align}
\nabla_x \mathcal{L} = \left( \partial_x f(x_0) \right)^T \nabla_y \mathcal{L},
\end{align}
where $\nabla_x \mathcal{L} \in \mathbb{R}^{d_1}$ is the gradient of the loss $\mathcal{L}$ with respect to the input $x$, $\partial_x f(x_0) \in \mathbb{R}^{d_2 \times d_1}$ is the Jacobian matrix of $f$ evaluated at $x_0$, and $\nabla_y \mathcal{L} \in \mathbb{R}^{d_2}$ is the gradient of the loss with respect to the output $y$.

\subsection{Minimum Cut Problem (MCP)} 
\noindent The Minimum Cut Problem aims to minimize the total weight of edges between two vertex subsets of a graph. Specifically, given a graph \( G = (X, E) \), where \( X \) and \( E \) denote the set of vertices and edges, respectively, each edge \( (u, v) \in E \) is assigned a weight \( w(u, v) \). A cut in the graph partitions \( X \) into two disjoint subsets \( \Gamma \) and \( X \setminus \Gamma \), such that each vertex belongs to exactly one subset. The weight of the cut, i.e., the sum of the weights of all edges that have one endpoint in \( \Gamma \) and the other endpoint in \( X \setminus \Gamma \), is defined as follows:
\begin{align}
\text{min}\sum\nolimits_{(u, v) \in E, u \in \Gamma, v \in X \setminus \Gamma} w(u, v).
\end{align}
In this work, we adopt the classic Stoer-Wagner algorithm \cite{stoer1997simple} to identify the optimal cuts for all GNN backbone layers except the first, due to its efficiency and robustness.
 
\section{Methods}

\noindent This section first describes the supernet architecture, followed by the overall search framework of the proposed \textbf{UGAS}, and concludes with the novel supernet partition method, \textbf{GC}.

\vspace{-0.2cm}
\subsection{Supernet Architecture}
\label{sec4.1}

\begin{figure}[!t]
\centering

\includegraphics[width=0.8\columnwidth]{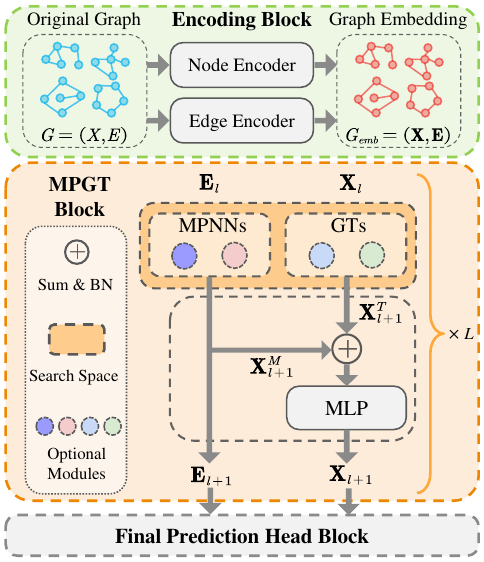} 
\caption{Illustration of the supernet architecture.}
\label{MPGT}
\vspace{-0.4cm}
\end{figure}
\noindent As shown in Figure~\ref{MPGT}, the supernet $\mathcal{A}$ is made up of three blocks, namely the encoding block, backbone block, and prediction head block. The encoding block encodes the initial features of nodes and edges. Then, the backbone block comprises multiple stacked Message Passing and Graph Transformers (MPGT) layers, thus referred to as the MPGT block. Finally, the prediction head block maps the processed features to the task-specific outputs.

\subsubsection{\textbf{Encoding Block}}
To achieve better performance, we adopt the encoding block proposed in the prior study \cite{GraphGPS}. Formally, given an input graph $G$ with \(N_G\) nodes and \(E_G\) edges, the encoding block models it as \( G_{emb} = (\mathbf{X}, \mathbf{E}) \), where \( \mathbf{X} \in\)  \( \mathbb{R}^{N_G \times d} \) and \( \mathbf{E} \in \)  \( \mathbb{R}^{E_G \times d} \) denote the features of nodes and edges, respectively. Here, \( d \) denotes the dimension of the feature.

\subsubsection{\textbf{MPGT Block}}

The entire MPGT block of supernet $\mathcal{A}$ comprises $L$ MPGT layers. Meanwhile, to search for architectures with high-level performance, each MPGT layer $\mathcal{M}_l$ contains the total number of $n$ distinct MPNN and GT modules, with each module adopted only once within the layer.
\begin{align}
\label{eq:supernet}
\mathcal{A} &= \{\mathcal{M}_1, \mathcal{M}_2, \ldots, \mathcal{M}_L\},
\\
\mathcal{M}_l &= \{\mathcal{M}_{l,1}, \mathcal{M}_{l,2}, \ldots, \mathcal{M}_{l,n}\}.
\end{align}
As shown in Table~\ref{tab:search_space}, the designed MPGT block comprises four MPNN modules and two GT modules. Note that this paper focuses exclusively on searching for the optimal combination of MPNNs and GTs; hence, the notations for encoding and prediction head blocks are omitted in Eq.~(\ref{eq:supernet}) for simplicity.
 
\begin{table}[!t]
\centering
\caption{The MPNN and GT modules adopted in this work}
\label{tab:search_space}
\resizebox{0.95\linewidth}{!}{
\begin{tabular}{c|c}
\hline 
MPNN & {GCN \cite{GCN}, GAT \cite{GAT}, GINE \cite{GINE}, and GENConv \cite{GENConv}} \\
\hline
GT & {Transformer \cite{Vaswani2017}, and Performer \cite{Performmer}} \\
\hline
\end{tabular}}
 \vspace{-0.2cm}
\end{table}
To effectively leverage the strengths of multiple MPNNs and GTs, we design a mechanism for fusing the outputs of these diversifying modules. Specifically, for the \(l\)th layer, the node feature $\mathbf{X}_{l+1}$ and edge feature $\mathbf{E}_{l+1}$ are computed as follows:
\begin{align}
\mathbf{X}^{M}_{l+1}, \mathbf{E}_{l+1} &= \frac{1}{n_1} \sum\nolimits_{i=1}^{n_1} \mathcal{M}_{l,i}\left(\mathbf{X}_l, \mathbf{E}_l \right), \\
\mathbf{X}^{T}_{l+1} &= \frac{1}{n-n_1} \sum\nolimits_{j=n_1+1}^{n} \mathcal{M}_{l,j}\left(\mathbf{X}_l\right), \\
\mathbf{X}_{l+1} &= \mathrm{MLP}_l\left(\mathbf{X}^{M}_{l+1} + \mathbf{X}^{T}_{l+1}\right),
\end{align}
where \( \mathbf{X}_l \in \mathbb{R}^{N_G \times d} \) and \( \mathbf{E}_l \in \mathbb{R}^{E_G \times d} \) denote the node and edge features at the $l$th layer, respectively. The variable \(n_1\) denotes the number of MPNNs used at the \(l\)th layer and the number of GT modules is $n-n_1$. The outputs, \(\mathbf{X}^{M}_{l+1}\) and \(\mathbf{X}^{T}_{l+1}\), are subsequently passed through a two-layer MLP to compute the layer-wise feature aggregation.

\subsubsection{\textbf{Prediction Head Block}}
The prediction head is responsible for mapping the graph embeddings produced by the MPGT Block, to the task-specific target space. In this work, the prediction head typically consists of MLPs, which compress high-dimensional graph embeddings into the specific output dimensions, e.g., the number of classes for classification tasks.
\begin{figure}[!t]
\centering
\includegraphics[width=0.95\columnwidth]{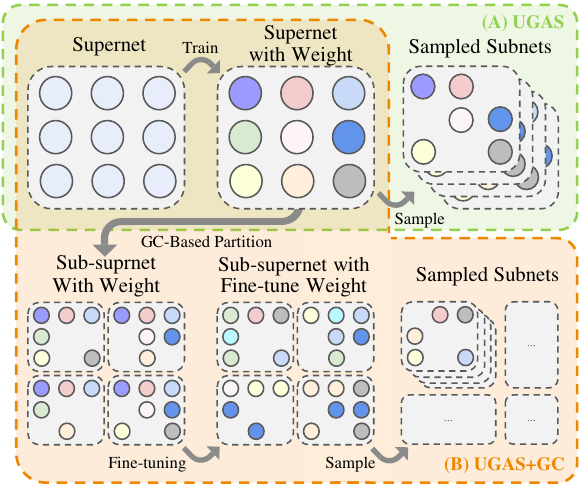}
\caption{Illustration of the subnet sampling processes from the supernet and sub-supernets, respectively. In subfigure (B), the supernet is partitioned into four sub-supernets ($k=2$) using the Gradient Contribution (GC) method.}
\label{UGAS framework}
  \vspace{-0.4cm}
\end{figure}
\subsection{UGAS Framework}
\label{sec4.2}
\noindent In this subsection, we outline the process of subnet sampling and the overall search process of UGAS, which is based on the Genetic Algorithm (GA). In addition, it is worth mentioning that UGAS is a general framework compatible with various NAS search paradigms. For details on its implementation using DARTS, please refer to Appendix~\ref{appendix:darts}.

\subsubsection{\textbf{Subnet}}
As mentioned in Section~\ref{sec4.1}, the search space (i.e., the supernet) consists of $L$ MPGT layers, each containing the number of $n$ distinct modules.  
The subnet $\mathcal{S}_i$ is sampled by randomly selecting certain modules from each layer of the supernet $\mathcal{A}$. Specifically, the subnet $\mathcal{S}_i$ is defined as follows:
\begin{align}
\mathcal{S}_i = \{\mathcal{M}_l' \mid \mathcal{M}_l' \in \mathcal{P}(\mathcal{M}_l), l\in \{1,2,\ldots,L\}\},
\end{align}
where $\mathcal{M}_l$ denotes the $l$th layer  of the supernet and $\mathcal{P}(\mathcal{M}_l)$ denotes the power set of  $\mathcal{M}_l$. As shown in Figure~\ref{UGAS framework}(A), we present certain subsets sampled from the supernet. In addition, it is worth mentioning that the supernet must be trained prior to the subset sampling process to obtain the weights of all modules.

\begin{algorithm}[!t]
\caption{The GA-based Search Process}
\label{alg:search}
\begin{flushleft}
\textbf{Input}: Supernet \( \mathcal{A} \), Population size \( P \),  Maximum iterations \( I \) \\
\textbf{Output}: The best subnet \( \mathcal{S}^{*} \)
\end{flushleft}
\begin{algorithmic}[1] 
\STATE \text{TrainSupernet}(\( \mathcal{A} \))
\STATE \( \mathcal{S} \leftarrow \text{RandomSample}(\mathcal{A}) \) \textcolor{gray}{\# the initial population $\mathcal{S}$}
\STATE \text{EvaluateFitness}(\( \mathcal{S} \))
\FOR{iteration \( i = 1 \) to \( I \)}
    \STATE \( \mathcal{S'} \leftarrow \text{Crossover} (\mathcal{S}) \cup  \text{ Mutation} (\mathcal{S}) \)
    \STATE \text{EvaluateFitness}(\( \mathcal{S'} \))
    \STATE  $\mathcal{S} \leftarrow $\text{CompetitiveSurvival}(\( \mathcal{S'}  \cup \mathcal{S}\))
\ENDFOR
\STATE \( \mathcal{S}^{\text{elite}} \leftarrow \text{SelectElite}(\mathcal{S}) \)
\STATE \( \mathcal{S}^* \leftarrow \text{Retrain}(\mathcal{S}^{elite}) \)
\\ \textcolor{gray}{\# Retrain the elite subnets to identify the best one}
\RETURN \( \mathcal{S}^{*} \)
\end{algorithmic}
\end{algorithm}
\subsubsection{\textbf{GA-based Search Process}}
This work employs a standard GA to search for subnets, consisting of initialization, selection, crossover, and mutation steps.

During the initialization process, UGAS randomly samples the number of $P$ subnets from the search space (i.e., the trained supernet) as the initial population, which is defined as follows:
\begin{align}
\mathcal{S}=\{\mathcal{S}_1, \mathcal{S}_2, \ldots, \mathcal{S}_P\}.
\end{align}
To maintain the diversity of the population, UGAS directly exploits all individuals in the current population as parents for the crossover and mutation operators. Subsequently, each operator generates the number of $P$ offspring.

Specifically, given two parents \(\mathcal{S}_a\) and \(\mathcal{S}_b\), for each layer, the crossover operator randomly selects a crossover point $c$, where $c\in\{1,...,n\}$, and then swaps the modules behind $c$ to provide two new offspring $\mathcal{S'}_a$ and $\mathcal{S'}_b$, defined as follows:
\begin{align}
\mathcal{S'}_{a,l} = \mathcal{S}_{a,l,1:c} \cup \mathcal{S}_{b,l,c+1:n},\\
\mathcal{S'}_{b,l} = \mathcal{S}_{b,l,1:c} \cup \mathcal{S}_{a,l,c+1:n}.
\end{align}
In addition, we set the probability of crossover occurring at each layer to \(p_c\) (see Appendix~\ref{appendix:Hyperparameter} for detailed parameter settings).

Given a parent subnet \(\mathcal{S}_a\), the mutation operator randomly flips one module in each layer. Specifically, if the parent subnet does not include this module, it is added to the child subnet; otherwise, it is removed. Additionally, the mutation probability for each layer is set to \(p_m\).

Subsequently, we directly evaluate the performance of the child individuals using their inherited weights. Only the top-$P$ individuals with the highest fitness values survive to the next iteration. The fitness function is defined as the metric corresponding to the downstream task, e.g., accuracy. When the maximum number of iterations is reached, we retrain the top five elite individuals with high fitness values from scratch to accurately assess their performance and ultimately retain the best individuals. Algorithm~\ref{alg:search} presents the complete GA-based search process.

\subsection{Gradient Contribution (GC)}
\label{sec4.3}

\noindent Different from prior studies \cite{Zhao2021, Hu2022}, we propose a novel perspective for analyzing the weight coupling issue. As shown in the gradient backpropagation paths of the computation graph in  Figure~\ref{Supernet}(B), the weights of modules in the preceding layer are updated by the gradients of all modules in the succeeding layer. It is inevitable that certain gradient directions may conflict, affecting the gradient descent trajectories and potentially causing the inherited subnets' weights to deviate from their true values that are trained from scratch.

Therefore, it is crucial to assign modules with conflicting gradients to distinct sub-supernets. To this end, we propose the efficient supernet partitioning method Gradient Contribution (GC).  Specifically, for two adjacent layers in the supernet, let \(\mathbf{y}_{l-1} = [\mathbf{y}_{l-1,1}, ..., \mathbf{y}_{l-1,n}] \) and \(\mathbf{y}_{l} = [\mathbf{y}_{l,1}, ..., \mathbf{y}_{l,n}]\) denote the output vectors of the $l-1$th and $l$th layers, respectively, where $l \in \{2,...,L\}$, $\mathbf{y}_{l-1,i}\in\mathbb{R}^{d_{i}}$, and $\mathbf{y}_{l,j}\in\mathbb{R}^{d_{j}}$. Note that all layers in the supernet share the same number of modules $n$. To explicitly distinguish modules between adjacent layers, we use $i$ to index modules in the $l-1$th layer and $j$ for those in the $l$th layer.

Therefore, the transformation from the output vector \( \mathbf{y}_{l-1} \) to the output vector \( \mathbf{y}_{l} \) is defined as follows:
\begin{equation}
    \mathbf{y}_{l} = f_{l}(\mathbf{y}_{l-1}) = [f_{l,1}(\mathbf{y}_{l-1}), f_{l,2}(\mathbf{y}_{l-1}), \ldots, f_{l,n}(\mathbf{y}_{l-1})],  
\end{equation}
where $f_{l,i}$ denotes the transformation function of the $i$th module. Assuming that \( \mathcal{L} \) denotes the loss function, the gradient with respect to \( \mathbf{y}_{l-1} \) can be expressed via the VJP as:  
\begin{align}
\nabla_{\mathbf{y}_{l-1}}\mathcal{L}=\left( \partial_{\mathbf{y}_{l-1}} f_{l}(\mathbf{y}_{l-1}) \right)^T \nabla_{\mathbf{y}_{l}} \mathcal{L}.
\end{align}
The term \( \partial_{\mathbf{y}_{l-1}} f_{l}(\mathbf{y}_{l-1}) \) denotes the Jacobian matrix, defined as follows:  
\begin{align}
\partial_{\mathbf{y}_{l-1}} f_{l}(\mathbf{y}_{l-1}) = \begin{bmatrix}
J_{l,1,1} & J_{l,1,2} & \ldots & J_{l,1,n} \\
J_{l,2,1} & J_{l,2,2} & \ldots & J_{l,2,n} \\
\vdots & \vdots & \ddots & \vdots \\
J_{l,n,1} & J_{l,n,2} & \ldots & J_{l,n,n}
\end{bmatrix} \in \mathbb{R}^{d_l \times d_{l-1}},
\end{align} 
 where \(d_l=\sum_{j=1}^n d_j\), \(d_{l-1}=\sum_{i=1}^n d_i\), and \( d_{j}\) and \( d_{i} \) denote the output dimensions of module \( \mathcal{M}_{l,j} \) and module \( \mathcal{M}_{l-1,i} \), respectively. Each Jacobian block matrix  \( J_{l,j,i} \) is defined as:  
\begin{align}
J_{l,j,i} =  \partial f_{l,j} \big/ \partial \mathbf{y}_{l-1,i}\in \mathbb{R}^{d_{j} \times d_{i}}.
\end{align} 
The gradient contribution of the \(j\)th module's output \(\mathbf{y}_{l,j}\) to \(\mathbf{y}_{l-1}\) is essentially the part of \(\nabla_{\mathbf{y}_{l-1}} \mathcal{L}\) contributed by \(\mathbf{y}_{l,j}\). According to the chain rule, this contribution \(C_{l,j}\) is defined as:
\begin{align}
C_{l,j} = \left( \frac{\partial \mathbf{y}_{l,j}}{\partial \mathbf{y}_{l-1}} \right)^\top \nabla_{\mathbf{y}_{l,j}} \mathcal{L} \in \mathbb{R}^{d_{l-1} \times 1}.
\end{align}
The term \(\frac{\partial \mathbf{y}_{l,j}}{\partial \mathbf{y}_{l-1}}\) denotes the block of the Jacobian matrix corresponding to the \(j\)th module, formed by horizontally concatenating: 
\begin{align}
\frac{\partial \mathbf{y}_{l,j}}{\partial \mathbf{y}_{l-1}} = \begin{bmatrix} J_{l,j,1} & J_{l,j,2} & \cdots & J_{l,j,n} \end{bmatrix} \in \mathbb{R}^{d_j \times d_{l-1}}.
\end{align}
This method provides a modular approach to decomposing the computation graph in complex networks, allowing for precise analysis of gradient contributions between modules.

\begin{algorithm}[!t]
\caption{The GC-based Supernet Partitioning Process}
\label{alg:enhanced_supernet_partition}
\begin{flushleft}
\textbf{Input}: Supernet \(\mathcal{A}\), number of partition points \( k \) \\
\textbf{Output}: Set of partitioned sub-supernets \(\mathcal{\bar{A}}\) \\
\end{flushleft}
\begin{algorithmic}[1] 
\FOR{\( l = 2 \) to \( L \)}
    \STATE Compute the  similarity \( S_{l,i,j},  \forall i, j \in \{1,...,n\}, \, i \neq j \)
    \STATE Compute the partition score \( \gamma_l \) 
    \STATE Determine the partition scheme \( \mathcal{P}_{l} =\{\Gamma_l, \mathcal{M}_l \setminus \Gamma_l\} \) 
\ENDFOR
\STATE Record the top-\(k\) schemes \(\mathcal{P}^{\text{top-}k}\) with the smallest scores
\STATE Generate all \( 2^k \) possible combinations of module sets
\FOR{\( i = 1 \) \textbf{to} \( 2^k \)}
\STATE Generate sub-supernet $\mathcal{\bar{A}}_i$ based on the \( i \)th combination
\ENDFOR
\STATE Retrain all sub-supernets
\STATE \textbf{return} \( \mathcal{\bar{A}} \)
\end{algorithmic}
\end{algorithm}

\subsection{GC-Based Supernet Partition}

\noindent As mentioned in Section~\ref{sec4.3}, this work uses $[C_{l,1},C_{l,2} \ldots, C_{l,n}]$ to measure the gradient contribution from $\mathcal{M}_{l} = \{\mathcal{M}_{l,1}, \ldots, \mathcal{M}_{l,n}\}$ to all modules in $\mathcal{M}_{l-1}$. Subsequently, the cosine similarity $S_{l,i,j}$ between two gradient contributions \( C_{l,i} \) and \( C_{l,j} \) is computed as follows:
\begin{align}
S_{l,i,j} =\frac{ C_{l,i} \cdot C_{l,j}} {\|C_{l,i}\| \|C_{l,j}\|}.
\end{align}
By computing the gradient contributions for all modules in \(\mathcal{M}_l\) and evaluating their mutual similarity, the modules with high similarity are partitioned into the same sub-supernet; otherwise, they are allocated to distinct sub-supernets. Specifically, to achieve this, GC treats all the $l$th layer modules as nodes and the cosine similarity values as the edge weights between nodes, formulating the task as a Minimum Cut Problem (MCP), whose objective is to minimize the cut weight (i.e., partition score \(\gamma_l\)), defined as follows:
\begin{align}
\gamma_{l} = \min_{\Gamma_{l} \subseteq \mathcal{M}_{l}} \sum\nolimits_{i \in \Gamma_{l}, j \in \mathcal{M}_{l} \setminus \Gamma_{l}} S_{l,i,j}.
\end{align}
To find the cut with the smallest cut weight for the $l$th layer, we adopt the classic MCP solver, Stoer-Wagner algorithm \cite{stoer1997simple} in all experiments conducted in this work. Therefore, all modules of the $l$th layer are divided into two sets, i.e.,$\Gamma_{l}$ and $\mathcal{M}_{l} \setminus \Gamma_{l}$.

\begin{figure}[!t]
\centering
\includegraphics[width=1\columnwidth]{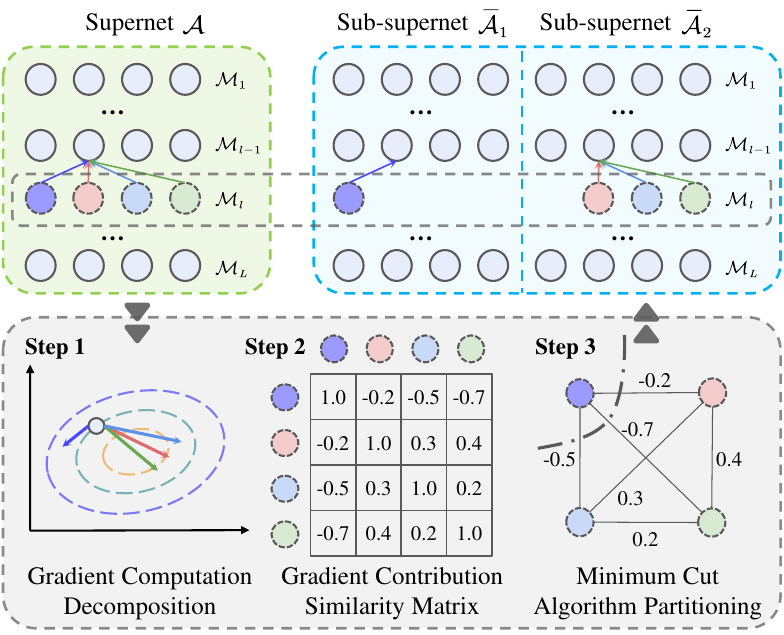} 
\caption{Illustration of the GC-based supernet partitioning process. All modules of the $l$th layer are partitioned into two sets based on the cosine similarity of their gradient contributions, corresponding to the $l$th layers of two sub-supernets (assuming $k=1$).}
\label{gc}
\vspace{-0.5cm}
\end{figure}

In this way, GC partitions the modules of each layer from 2 to $L$ into two sets. Finally, the top-\( k \) layers with the smallest \( \gamma_{l} \) are selected as the layers that need to be partitioned,  $ k \in \{1, 2,..., L-1\}$. For each of the selected layers, there are two available module sets (i.e., $\Gamma_{l}$ and $\mathcal{M}_{l} \setminus \Gamma_{l}$), ultimately resulting in a total of $2^k$ distinct sub-supernets. The pseudocode of GC is presented in Algorithm~\ref{alg:enhanced_supernet_partition}. As shown in Figure~\ref{gc}, we present the supernet partitioning process (assuming the supernet needs to be partitioned into two supernets, i.e., $k = 1$). The supernet \( \mathcal{A} \) is partitioned into two sub-supernets, i.e., $\mathcal{\bar{A}}_1$ and $\mathcal{\bar{A}}_2$, where $\mathcal{\bar{A}}_1 = \{\mathcal{M}_{1}, \ldots, \Gamma_{l}, \ldots, \mathcal{M}_{L}\}$ and $\mathcal{\bar{A}}_2 = \{\mathcal{M}_{1}, \ldots, \mathcal{M}_{l} \setminus \Gamma_{l}, \ldots, \mathcal{M}_{L}\}$. In addition, Appendix~\ref{secproof} provides a proof that partitioning based on gradient directions effectively mitigates the weight coupling problem.

Although GM \cite{Hu2022} also partitions the supernet based on the gradients between modules, its perspective on analyzing the weight coupling problem fundamentally differs from ours. Please see Appendix~\ref{appendix:gm} for details about the differences in the partitioning processes between GC and GM. Section~\ref{sec5.3} presents the comprehensive performance comparison between GC and GM.

\section{Experiments}
\label{sec5}
\noindent This section first presents the experimental setups, followed by extensive experimental results on various GNN datasets across different metrics. These include comparisons with the hand-designed GNN architectures, GNAS methods, and other weight-sharing NAS methods, as well as ablation studies. Please see Appendixes~\ref{appendixf} to \ref{appendixh} for more experimental results about the rationale for supernet partitioning based on early-stage gradients, the rationale for incorporating GC into UGAS, and the performance comparison of different minimum cut algorithms.

\begin{table*}[!t]
\centering
\caption{Performance comparison with manually designed GNN architectures}
\scalebox{0.85} {
\label{tab:performance_models}
\begin{tabular}{l|cccc|ccc}
\hline
\multirow{3}{*}{\textbf{Methods}} & \multicolumn{4}{c|}{\textbf{BGNN} \cite{Benchmark_GNN}} & \multicolumn{3}{c}{\textbf{LRGB} \cite{LRGB}}\\ 
\cline{2-8} 
& \textbf{Mnist} & \textbf{Cifar10} & \textbf{Pattern} & \textbf{Cluster} & \textbf{PascalVOC-SP} & \textbf{Peptides-func} & \textbf{Peptides-struct}\\
\cline{2-8}
& \textbf{Accuracy} $\uparrow$ & \textbf{Accuracy} $\uparrow$ & \textbf{Accuracy} $\uparrow$ & \textbf{Accuracy} $\uparrow$ & \textbf{F1 score} $\uparrow$ & \textbf{AP} $\uparrow$ & \textbf{MAE} $\downarrow$ \\
\hline
GCN (2017) \cite{GCN} $\dagger$ & $90.71 \pm 0.22$ & $55.71 \pm 0.38$ & $71.89 \pm 0.33$ & $68.50 \pm 0.98$ & $0.127 \pm 0.006$ & $0.593 \pm 0.002$ & $0.350 \pm 0.001$\\
GatedGCN (2017) \cite{GatedGCN} $\dagger$ & $97.34 \pm 0.14$ & $67.31 \pm 0.31$ & $85.57 \pm 0.09$ & $73.84 \pm 0.33$ & $0.287 \pm 0.022$ & $0.586 \pm 0.008$ & $0.342 \pm 0.001$ \\
GAT (2018) \cite{GAT}  $\dagger$& $95.54 \pm 0.21$ & $64.22 \pm 0.46$ & $78.27 \pm 0.19$ & $70.59 \pm 0.45$ &-&-&-\\
GIN (2019)\cite{GIN}  $\dagger$& $96.49 \pm 0.25$ & $55.26 \pm 1.53$ & $85.39 \pm 0.14$ & $64.72 \pm 1.55$ &-&-&-\\
GINE (2020) \cite{GINE} $\dagger$& - & - & - & - & $0.127 \pm 0.008$ & $0.550 \pm 0.008$ & $0.356 \pm 0.005$ \\
PNA (2020) \cite{PNA}  $\dagger$& $97.94 \pm 0.12$ & $70.35 \pm 0.63$ & - & - &-&-&-\\
DGN (2021) \cite{DGN}  $\dagger$& - & $72.84 \pm 0.42$ & $86.68 \pm 0.03$ & - &-&-&-\\
SAN (2021) \cite{SAN}  $\dagger$& - & - & $86.58 \pm 0.04$ & $76.69 \pm 0.65$ & $0.323 \pm 0.004$ & $0.638 \pm 0.012$ & $0.268 \pm 0.004$ \\
K-Subgraph SAT (2022) \cite{K_Subgraph_SAT} $\dagger$ & - & - & $86.85 \pm 0.04$ & $77.86 \pm 0.10$ &-&-&-\\
EGT (2022) \cite{EGT}  $\dagger$& $98.17 \pm 0.09$ & $68.70 \pm 0.41$ & $86.82 \pm 0.02$  &$\textbf{79.23} \boldsymbol{\pm} \textbf{0.35}$ &-&-&-\\
GraphGPS (2022) \cite{GraphGPS}  $\dagger$& $98.05 \pm 0.13$ & $72.30 \pm 0.36$ & $86.69 \pm 0.06$ & $78.02 \pm 0.18$ & $0.375 \pm 0.011$ & $0.654 \pm 0.004$ & $0.250 \pm 0.001$\\
TIGT (2024) \cite{TIGT}  $\dagger$& $98.23 \pm 0.13$ & $73.96 \pm 0.36$ & $86.68 \pm 0.06$ & $78.03 \pm 0.22$ &-&-&-\\
\hline
UGAS+OS (GA) & $98.39 \pm 0.15$ & $73.57 \pm 0.16$ & $86.00 \pm 0.05$ & $77.42 \pm 0.39$ & $0.400 \pm 0.006$ & $0.651 \pm 0.005$ & $\underline{0.250 \pm 0.003}$\\
UGAS+GC (GA) & $\textbf{98.50} \boldsymbol{\pm} \textbf{0.14}$ & $\textbf{74.14} \boldsymbol{\pm} \textbf{0.29}$ & $\textbf{86.89} \boldsymbol{\pm} \textbf{0.02}$ & $\underline{78.14 \pm 0.21}$ & $\underline{0.418 \pm 0.009}$ & $\underline{0.661 \pm 0.013}$ & $\textbf{0.247} \boldsymbol{\pm} \textbf{0.002}$\\
UGAS+OS (DARTS) & $98.23 \pm 0.12$ & $73.78 \pm 0.46$ & $86.59 \pm 0.09$ & $75.30 \pm 0.31$ & $\textbf{0.426} \boldsymbol{\pm} \textbf{0.015}$ & $0.659 \pm 0.010$ & $0.252 \pm 0.001$\\ 
UGAS+GC (DARTS)& \underline{$98.40 \pm 0.20$} & $\underline{74.01 \pm 0.36}$ & $\underline{86.74 \pm 0.05}$ & $76.87 \pm 0.21$ & $0.418 \pm 0.023$ & $\textbf{0.667} \boldsymbol{\pm} \textbf{0.007}$ & $0.251 \pm 0.000$\\ 
\hline
\end{tabular}}
\begin{tablenotes}
\item[] \ \ \ \ \ Note: Symbol $\dagger$ denotes that the results are taken from the original manuscript, while the rest are obtained in this work.
\end{tablenotes}
 \vspace{-0.4cm}
\end{table*}

\subsection{Experimental Setups}


\subsubsection{\textbf{Datasets}}
We select a range of graph datasets from various benchmarks to evaluate the performance of UGAS and GC. These include the Cifar10, Mnist, Pattern, and Cluster datasets from the Benchmarking Graph Neural Networks (BGNN) \cite{Benchmark_GNN}, the PascalVOC-SP, Peptides-Func, and Peptides-Struct datasets from the Long-Range Graph Benchmark (LRGB) \cite{LRGB}, and the Molhiv, Molsider, and Molbace datasets from the Open Graph Benchmark (OGB) for graph property prediction \cite{hu2020ogb}.
\subsubsection{\textbf{Comparison Methods}}

To comprehensively assess the performance of GC and UGAS, we choose extensive benchmarking methods for comparison, including twelve human-designed GNN architectures, namely GCN \cite{GCN}, GINE \cite{GINE}, GIN \cite{GIN}, GAT \cite{GAT}, GatedGCN \cite{GatedGCN}, PNA \cite{PNA}, DGN \cite{DGN}, SAN \cite{SAN}, K-Subgraph SAT \cite{K_Subgraph_SAT}, EGT \cite{EGT}, GraphGPS \cite{GraphGPS}, and TIGT \cite{TIGT}; five GNAS methods, namely PAS \cite{PAS}, GRACES \cite{GRACES}, AutoGT \cite{AutoGT}, MTGC \cite{MTGC}, and DCGAS \cite{DCGAS}; and five weight-sharing NAS methods, namely OS \cite{Pham2018}, FS \cite{Zhao2021}, GM \cite{Hu2022}, LID \cite{He2023}, and EFS \cite{SBS}.
\subsubsection{\textbf{Implementation Setups}}

In this paper, we implement UGAS using two architecture search methods, namely GA and DARTS \cite{Liu2019}. The parameter settings for GA and UGAS are provided in Appendix~\ref{appendix:Hyperparameter}, while Appendix~\ref{appendix:darts} presents a detailed introduction to DARTS. Additionally, the subnets searched by both GA and DARTS need to be retrained from scratch before the final evaluation. For GA, we select the top-$5$ architectures for retraining, with the final result derived from the best-performing architectures among them, While for DARTS, we retain the same number of subnets as the (sub-)supernet for retraining. The parameter settings for GC are provided in Appendix~\ref{appendix:Hyperparameter}. Finally, all experiments are executed on a computational environment equipped with a single NVIDIA A800 GPU with 80GB of memory.

\subsection{Comparison with the Hand-designed GNNs}

\noindent In this subsection, we compare the UGAS-searched architectures with twelve manually designed GNN architectures using datasets from BGNN \cite{Benchmark_GNN} and LRGB \cite{LRGB} in Table~\ref{tab:performance_models}. BGNN tests the architecture’s ability to extract local features from small to medium-sized graphs, while LRGB evaluates its capacity to capture distant node relationships in large graphs. The experimental results on BGNN demonstrate that, whether implemented based on GA or DARTS, the UGAS-searched architectures outperform the manually designed architectures on BGNN, owing to the proposed MPGT joint search framework. For example, compared to TIGT \cite{TIGT}, the average accuracy of architectures searched by UGAS+GC (GA) yields average accuracy improvements of 0.27\%, 0.18\%, 0.21\% and 0.11\% across the four datasets, respectively. On LRGB, compared to GraphGPS \cite{GraphGPS}, the architectures searched by UGAS+GC (DARTS) achieve average performance metrics improvements of 0.043 and 0.013 on the PascalVOC-SP and Peptides-func datasets, respectively. By integrating both MPNN and GT modules, our method not only achieves robust performance on small-scale graphs but also excels on large-scale graphs.

\begin{table}[!t]
\centering
\caption{Performance comparison with NAS methods}
\label{tab:compareGNAS}
\scalebox{0.9} {
\begin{tabular}{l|ccc}
\hline
\multirow{2}{*}{\textbf{Methods}} & \textbf{Molhiv} & \textbf{Molsider} & \textbf{Molbace} \\
\cline{2-4}
& \textbf{ROC-AUC} $\uparrow$ & \textbf{ROC-AUC} $\uparrow$ & \textbf{ROC-AUC} $\uparrow$\\
\hline
PAS (CIKM'21)  $\dagger$ & $71.19 \pm 2.28$ & $59.31 \pm 1.48$ & $76.59 \pm 1.87$ \\
GRACES (ICML'22)  $\dagger$ & $77.31 \pm 1.00$ & $61.85 \pm 2.56$ & $79.46 \pm 3.04$ \\
AutoGT (ICLR'23)  $\dagger$ & $74.95 \pm 1.02$ & - & $76.70 \pm 1.42$ \\
MTGC (NeurIPS'23)  $\dagger$ & - & $62.08 \pm 1.76$ & - \\
DCGAS (AAAI'24)  $\dagger$ & $78.04 \pm 0.71$ & $\textbf{63.46} \boldsymbol{\pm} \textbf{1.42}$ & $81.31 \pm 1.94$ \\
\hline
UGAS+OS (ICML'18) $\ddagger$ & $77.94 \pm 1.38$ & $61.88 \pm 0.49$ & $81.12 \pm 0.65$ \\
UGAS+FS (ICML'21) $\ddagger$ & $79.15 \pm 1.13$ & $62.71 \pm 0.50$ & $83.41 \pm 0.93$ \\
UGAS+GM (ICLR'22) $\ddagger$ & $79.34 \pm 0.36$ & $63.12 \pm 0.98$ & $83.62 \pm 1.20$ \\
UGAS+LID (AAAI'23) $\ddagger$ & $79.32 \pm 0.43$ & $61.75 \pm 0.38$ & $84.16 \pm 1.49$ \\
UGAS+EFS (AAAI'25) $\ddagger$ & $78.53 \pm 0.22$ & $63.36 \pm 0.55$ & $84.32 \pm 1.63$ \\
\hline
UGAS+GC (GA) & $\underline{79.95 \pm 0.70}$ & $\underline{63.42 \pm 0.71}$ & $\textbf{85.79} \boldsymbol{\pm} \textbf{1.91}$ \\
UGAS+GC (DARTS) & $\textbf{79.97} \boldsymbol{\pm} \textbf{0.51}$ & $62.31 \pm 0.90$ & $\underline{84.63\pm 1.21}$ \\
\hline
\end{tabular}}
\begin{tablenotes}
\item[] 
Note: Symbol $\ddagger$ denotes the results obtained from this work.
\end{tablenotes}
 \vspace{-0.2cm}
\end{table}

\subsection{Comparison with NAS Methods}
\label{sec5.3}

\subsubsection{\textbf{Performance Comparison with NAS methods}}
\label{5.3.1}

In this subsection, to assess the performance of UGAS, we conduct extensive experiments comparing it with existing GNAS methods, including four MPNN-based GNAS methods, namely PAS~\cite{PAS}, GRACES~\cite{GRACES}, MTGC~\cite{MTGC}, and DCGAS~\cite{DCGAS}, as well as one GT-based GNAS method AutoGT~\cite{AutoGT}, on three OGB datasets, namely Molhiv, Molsider, and Molbace. In addition, we include five weight-sharing NAS methods (implemented based on UGAS+GA with the same parameter settings as GC) to further demonstrate the effectiveness of GC. As shown in Table~\ref{tab:compareGNAS}, when compared to other GNAS methods, both UGAS+GC (GA) and UGAS+GC (DARTS) achieve high-level performance and attain higher ROC-AUC scores on two of the three datasets. For example, UGAS+GC (GA) yields a 4.48\% absolute improvement on Molbace over DCGAS. Furthermore, UGAS+GC (GA) outperforms all other weight-sharing NAS methods. This can be attributed to GC's ability to assign modules with conflicting gradient directions to distinct sub-supernets (see the following paragraph for a more in-depth analysis). Finally, regardless of the weight-sharing method adopted, UGAS outperforms other GNAS methods that search only MPNNs or GTs, further validating the effectiveness of the proposed MPGT block. 
\begin{table}[!t]
\centering
\caption{Ranking correlation comparison with weight-sharing NAS methods on the Cluster dataset}
\label{tab:correlation}
\scalebox{0.9} {\begin{tabular}{l|cc}
\hline
\textbf{Methods} & \textbf{Kendall} & \textbf{Spearman} \\
\hline
UGAS+OS (ICML'18) & $0.37 \pm 0.17$ & $0.50 \pm 0.20$ \\
UGAS+FS (ICML'21) & $0.36 \pm 0.19$ & $0.49 \pm 0.22$ \\
UGAS+GM (ICLR'22) & $0.42 \pm 0.22$ & $0.61 \pm 0.16$ \\
UGAS+LID (AAAI'23)& $\underline{0.45 \pm 0.14}$ & $\underline{0.63 \pm 0.16}$ \\
UGAS+EFS (AAAI'25)& $0.43 \pm 0.20$ & $0.57 \pm 0.19$ \\
\hline
UGAS+GC & $\textbf{0.54} \boldsymbol{\pm} \textbf{0.17}$ & $\textbf{0.73} \boldsymbol{\pm} \textbf{0.19}$ \\
\hline
\end{tabular}}
\begin{tablenotes}
\item[] 
Note: the results obtained from the implementation using GA.
\end{tablenotes}
 \vspace{-0.4cm}
\end{table}

\subsubsection{\textbf{Ranking Correlation Comparison}} 
\label{5.3.2}
As mentioned in Section~\ref{5.3.1}, weight-sharing NAS methods play a crucial role in determining the performance of the final searched subnets. A plausible explanation is that certain weight-sharing methods exhibit the weight coupling problem, leading to substantial differences between the weights inherited by the subnet from the (sub-)supernet and those obtained when training the subnet from scratch. To validate this, we compute the Kendall and Spearman rank correlations between the accuracy differences of the same subnet when evaluated with weights trained from scratch versus those inherited from the (sub-)supernet.  A stronger correlation indicates greater similarity between the weights trained from scratch and the inherited weights (see Appendix~\ref{appendix_rank} for the detailed experimental setup).  As shown in Table~\ref{tab:correlation}, GC outperforms the other weight-sharing methods, achieving the highest Kendall and Spearman correlation scores (i.e., a Spearman score of 0.73). The superior performance can be attributed to GC's ability to assign modules with conflicting gradient directions to distinct sub-supernets, ensuring that the weights inherited by subnets closely resemble those obtained through training from scratch. Consequently, high-performance subnets are thoroughly explored, while poor-performing subnets are neither overestimated nor retained during the search process. Other weight-sharing methods may overestimate the performance of certain poor-performing subnets, leading to deviated search directions and ultimately resulting in suboptimal subnets.

\begin{figure}[!t]
\centering
\includegraphics[width=0.9\columnwidth]{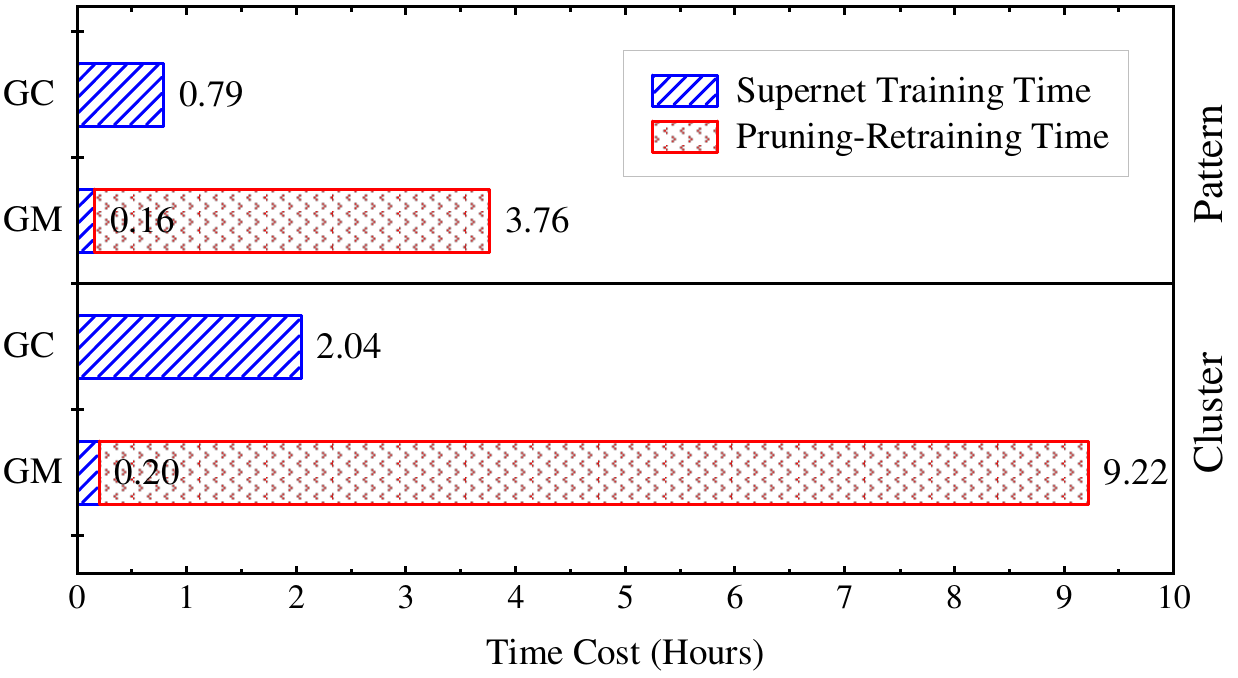} 
\caption{Comparison of supernet partitioning time between GC and GM on Pattern and Cluster datasets.}
 \vspace{-0.2cm}
\label{fig:timecost}
\end{figure}

\subsubsection{\textbf{Time Efficiency Comparison with GM}}

Among the weight-sharing NAS methods discussed in Section~\ref{5.3.2}, GM is the most similar to our proposed GC, as both leverage the gradient-based partitioning principle. Therefore, in Figure~\ref{fig:timecost}, we compare the time required by GC and GM to partition the supernet on the Pattern and Cluster datasets. Generally, the time required for GC to partition a supernet is primarily determined by the supernet training time. In contrast, GM's partitioning time encompasses not only the supernet training time but also the additional time required for pruning and retraining numerous temporary supernets. As shown in Figure~\ref{fig:timecost}, the supernet training time of GC is approximately four times that of GM, owing to its need to retain the computation graphs from the supernet backpropagation for decomposing VJP.  However, the extra overhead of pruning and retraining in GM leads to a significantly longer overall runtime. Notably, the overall supernet partitioning time for GC is approximately 21\% of that for GM. Furthermore, other few-shot NAS methods (e.g., LID \cite{He2023} and EFS \cite{SBS}) employ zero-cost proxy methods for supernet partitioning, thereby avoiding computational overhead associated with gradient computation (see Section~\ref{sec2.2} for more details). However, as evidenced by lower Kendall and Spearman rank correlation coefficients shown in Table~\ref{tab:correlation}, these methods have resulted in weaker consistency in accuracy when the same subnet adopts different weights (i.e., the weights trained from scratch versus the estimated weights inherited from the sub-supernet), potentially leading to inaccurate sub-supernet partitioning.
\subsection{Ablation Studies}
\begin{table}[!t]
\centering
\caption{Ablation study on different parameter $k$ settings}
\scalebox{0.8} {
\begin{tabular}{c|c|cc|c|c}
\hline
\multirow{2}{*}{\textbf{Methods}} & \textbf{\#Split} & \multicolumn{2}{c|}{\textbf{Acc. (\%)}} & \textbf{Avg. } & \textbf{Search Cost}  \\ 
\cline{3-4}
 & \textbf{($k$)}  & \textbf{Top 1} & \textbf{Avg. 5} & \textbf{Params (M)} & \textbf{(GPU Days)}  \\ 
\hline
w/o GC & - & 77.42 & 76.27 & 1.24 & 1.59  \\ 
\hline
\multirow{4}{*}{w/ GC} & 1  & 77.85 & 77.62 & 1.16 & 1.96 \\ 
     & 2 & \textbf{78.14} & \textbf{77.98} & 1.05 & 2.37 \\ 
     & 3  & 77.94 & 77.21 & 1.02 & 3.44   \\ 
     & 4  & 76.81 & 76.19 & 0.90 & 5.55  \\ 
\hline
\end{tabular}}
\label{tab:ablation_k}
 \vspace{-0.4cm}
\end{table}

\begin{table}[!t]
\centering
\caption{Ablation study on different design choices}
\label{tab:ablation_results}
\scalebox{0.85} {
\begin{tabular}{cc|ccc}
\hline
 \multirow{2}{*}{\textbf{MPNNs}} & \multirow{2}{*}{\textbf{GTs}} & \textbf{Cluster}& \textbf{Peptides-func} & \textbf{Peptides-struct}\\
\cline{3-5}
  & & \textbf{Accuracy} $\uparrow$ & \textbf{AP} $\uparrow$ & \textbf{MAE} $\downarrow$\\
\hline
 \ding{51} & \ding{55} & $76.45 \pm 0.19$ & $0.619 \pm 0.008$ & $0.262 \pm 0.002$\\
 \ding{55} & \ding{51} & $75.17 \pm 0.23$ & $0.632 \pm 0.007$ & $0.255 \pm 0.001$\\
 \ding{51} & \ding{51} & $\textbf{78.14} \pm \textbf{0.21}$ & $\textbf{0.661} \pm \textbf{0.013}$ & $\textbf{0.247 } \boldsymbol{}{\pm} \textbf{ 0.002}$\\
\hline
\end{tabular}}
 \vspace{-0.3cm}
\end{table}

\noindent In this section, we conduct ablation experiments to investigate the impact of different numbers of sub-supernets. Specifically, in Table~\ref{tab:ablation_k}, we present the accuracy of the best-performing subnet (Top 1) as well as the average accuracy and the number of parameters of the top five subnets (\mbox{Avg. 5}) under different values of $k$, following the prior study \cite{Hu2022}. As shown in Table~\ref{tab:ablation_k},
increasing $k$ yields a more fine-grained partition of the supernet, which in turn constrains the subnet search space and reduces the average number of parameters. In addition, finer partitioning necessitates training a larger number of sub-supernets and subnets, significantly extending the search time. When $k$ is set to 2 (i.e., resulting in four sub-supernets), UGAS+GA achieves optimal performance. A plausible reason is that when $k$ is set to 1, the weights of the subnet trained from scratch deviate from those inherited from the sub-supernets, preventing the discovery of well-performing subnets. Conversely, while higher weight similarity is maintained, setting $k$ to 3 or 4 further reduces the number of parameters, thereby diminishing the expressive capacity of the subnets.

In addition, we conduct ablation studies on the Cluster, Peptides-func, and Peptides-struct datasets to assess the effectiveness of the joint search framework that integrates both MPNN and GT modules. As shown in Table~\ref{tab:ablation_results}, architectures derived from the joint framework achieve superior performance compared to those searched using only one type of module. Notably, on the Cluster dataset, comprising small-scale graphs, the architectures searched solely from MPNNs outperform those based exclusively on GTs, reflecting the strength of MPNNs in capturing local structural patterns. Conversely, on the Peptides-func and Peptides-struct datasets, which exhibit long-range dependencies, architectures based exclusively on GT modules perform better. These experimental results further underscore the advantage of our MPGT framework in adapting to graphs of diverse sizes.
 
\section{Conclusion}
\noindent To better alleviate the weight coupling problem, we propose the GC method, which incorporates a novel perspective and efficiently partitions the supernet according to the gradient directions of the modules. In addition, we propose the generic framework UGAS, which searches for the optimal combination of MPNNs and GTs. The extensive experiments demonstrate the effectiveness of the proposed GC and UGAS. Future work could expand the set of optional modules to search for GNN architectures.

\section*{Acknowledgments}
\noindent This work was supported by the National Science and Technology Major Project under Grant 2021ZD0112500. We thank the Computing Center of Jilin Province for technical support.

\bibliographystyle{ACM-Reference-Format}
\bibliography{base}

\appendix
\section{Detailed Parameter Settings}
\label{appendix:Hyperparameter}
\noindent In this section, we provide a detailed overview of the main parameter settings for UGAS and GC employed across ten datasets.

\textbf{MPGT Block}: The MPGT Block is the core structure of the searched architectures. On most datasets, we set the number of MPGT layers $ L $ to be between four and ten following the configuration in the prior study \cite{GraphGPS}. For the Cluster dataset, setting $L$ to fifteen achieves better performance. In addition, across all experiments, the number of available modules per layer $n$ is set to six, comprising four MPNN and two GT modules.

\textbf{Supernet Partition}: 
For all experiments conducted in this paper, the number of partition points is set to two, meaning the supernet is partitioned into four (i.e., $ 2^k $) sub-supernets, unless otherwise specified in ablation studies.

\textbf{Genetic Algorithm (GA)}: In the GA, the maximum iterations $ I $ range from 5 to 20, and the population size $ P $ ranges from 20 to 50, adjusted according to different datasets. The crossover probability $ p_c $ and mutation probability $ p_m $ are set to 0.6 and 0.1, respectively. During each iteration, each operator generates $P$ offspring. Finally, from the combined pool of parent and offspring populations (total of $ 3 \times P $ individuals), the top-$ P $ individuals with the highest fitness are retained.

\textbf{Optimizer}: We employ the AdamW optimizer \cite{adamw} during the training of GNN architectures, with default settings of $ \beta_1 = 0.9 $, $ \beta_2 = 0.999 $, and $ \epsilon = 10^{-8} $. Training is initiated with a linear warm-up of the learning rate followed by its cosine decay.

\section{Proof of the Effectiveness of GC}
\label{secproof}
\noindent The prior study \cite{Hu2022} shows that gradient conflicts in supernets cause zigzag trajectories during stochastic gradient descent (SGD), which leads to biased performance predictions of subnets. To mitigate this distortion, GC aims to enhance internal average gradient contribution consistency, thereby improving the accuracy of subnet performance estimation. To prove GC’s ability to achieve this consistency, let $S_{\text{total}}$ denote the overall cosine similarity matrix of the modules in the $l$th layer ($l \in\{2,...,L\}$), which is defined as follows:
\begin{align}
 S_{\text{total}}   = \sum\nolimits_{(i,j) \in \mathcal{M} \times \mathcal{M}} S_{i,j},
\end{align}
where \(\mathcal{M}\) denotes the module set at the $l$th layer. After partitioning, $S_{\text{total}}$ is computed as follows:
\begin{align}
S_{\text{total}} &= S_{\Gamma} + S_{\mathcal{M} \setminus \Gamma} + 2S_{\text{cross}}, \Gamma \subseteq \mathcal{M}, \\
    S_{\Gamma}           &= \sum\nolimits_{(i,j) \in \Gamma \times \Gamma} S_{i,j}, \\
    S_{\mathcal{M} \setminus \Gamma} &= \sum\nolimits_{(i,j) \in (\mathcal{M} \setminus \Gamma) \times (\mathcal{M} \setminus \Gamma)} S_{i,j}, \\
    S_{\text{cross}}     &= \sum\nolimits_{(i,j) \in \Gamma \times (\mathcal{M} \setminus \Gamma)} S_{i,j},
\end{align}
\begin{proposition}
There exists a subset \(\Gamma \subseteq \mathcal{M}\) obtained by a gradient contribution-based minimum-cut algorithm such that the average gradient contribution similarity within \(\Gamma\) is greater than or equal to that of the original set:
\begin{align}
\frac{S_{\Gamma}}{|\Gamma|^2} \geq \frac{S_{\text{total}}}{|\mathcal{M}|^2}.
\end{align}
\end{proposition}
\begin{proof}

Let \(\Gamma^* = \arg\min_{\Gamma \subseteq \mathcal{M}} S_{\text{cross}}\) be the partition scheme obtained by the minimum cut algorithm that minimizes cross-group similarity. Let \(|\mathcal{M}| = n\), \(|\Gamma^*| = a\), \( n-a = b \), and \(  S_{\text{total}}/n^2 =\mu\). 

\noindent\textit{Proof by contradiction}: Suppose both subsets violate the proposition:
\begin{align}
\frac{S_{\Gamma^*}}{a^2} < \mu \quad \text{and} \quad \frac{S_{\mathcal{M}\setminus\Gamma^*}}{b^2} < \mu.
\end{align}
Hence,
\begin{align}
S_{\Gamma^*} + S_{\mathcal{M}\setminus\Gamma^*} < \mu(a^2 + b^2) = \mu(n^2 - 2ab).
\end{align}
Substituting \(S_{\text{total}} = S_{\Gamma^*} + S_{\mathcal{M}\setminus\Gamma^*} + 2S_{\text{cross}}^*\) yields:
\begin{align}
\label{eq:27}
S_{\text{cross}}^* > \mu ab.
\end{align}

The probability that two modules $i$ and $j$ are partitioned into distinct subsets is $\frac{ab}{n(n-1)}$. Hence, for the subset of size $a$, the expected cross-similarity $\mathbb{E}[S_{\text{cross}}] $ is defined as follows:
\begin{align}
\label{eq:28}
\mathbb{E}[S_{\text{cross}}] = \frac{ab}{n(n-1)} \left( S_{\text{total}} - \sum_i S_{i,i} \right).
\end{align}
Because \(\Gamma^*\) minimizes \(S_{\text{cross}}\), we have \(S_{\text{cross}}^* \leq \mathbb{E}[S_{\text{cross}}]\). Substituting this inequality into Eqs.~\eqref{eq:27} and \eqref{eq:28} gives
\begin{align}
\mu ab < \frac{ab}{n(n-1)} \left( S_{\text{total}} - \sum_i S_{i,i} \right).
\end{align}
Next, replace $\mu$ with \( S_{\text{total}}/n^2\) and simplify. This yields the contradiction:
\begin{align}
\label{eq:31}
n \sum_i S_{i,i} < S_{\text{total}}.
\end{align}
However, for gradient metrics, each self-similarity satisfies \(S_{i,i} = 1\), and each off-diagonal term satisfies \(S_{i,j} \leq 1\). Hence,
\begin{align}
\label{eq:32}
S_{\text{total}} = \underbrace{\sum_i S_{i,i}}_{= \sum_i 1 = n} + \sum_{i \neq j} S_{i,j} \leq n \sum_i S_{i,i}.
\end{align}
Eqs. \eqref{eq:31} and \eqref{eq:32} contradict one another. Thus, the partition induced by $\Gamma^*$ ensures that at least one of the subsets ($\Gamma^*$ or $\mathcal{M} \setminus \Gamma^*$) satisfies the inequality \(\frac{S_{\Gamma}}{|\Gamma|^2} \geq \mu\).
\end{proof}

\section{Implementation Details of DARTS}
\label{appendix:darts}
\noindent In Section~\ref{sec5}, we further assess the performance of UGAS using the DARTS search method. The core idea behind DARTS is to relax the discrete architecture search space into a continuous one by introducing learnable architecture parameters, which denote the importance weights of candidate MPNN and GT modules, respectively. These parameters are jointly optimized with the (sub-)supernet weights via gradient descent, enabling the differentiable architecture search. 

Within the MPGT layer, we implement DARTS by assigning learnable weights to each module. For the $l$th layer , the outputs of MPNN and GT modules are aggregated as follows:
\begin{align}
\mathbf{X}^{M}_{l+1}, \mathbf{E}_{l+1} &= \sum_{i=1}^{n_1} \frac{\exp(\alpha_{l,i})}{\sum_{k=1}^{n_1} \exp(\alpha_{l,k})} \times \mathcal{M}_{l,i}(\mathbf{X}_l, \mathbf{E}_l), \\
\mathbf{X}^{T}_{l+1} &= \sum_{j=n_1+1}^{n} \frac{\exp(\alpha_{l,j})}{\sum_{k=n_1+1}^{n} \exp(\alpha_{l,k})} \times \mathcal{M}_{l,j}(\mathbf{X}_l),
\end{align}
where $\alpha_{l,i}$ denotes the module weight for the $i$-th module. Parameters $n$ and $n_1$ denote the total number of modules and the number of MPNN modules, respectively. Notably, consistent with the GA search method, the outputs of the MPNN and GT modules must be computed separately. The softmax normalization ensures that the sum of the weights is 1, maintaining a probabilistic interpretation of module importance.

During the search phase, both module weights (i.e., $\boldsymbol{\alpha}$) and (sub-)supernet weights are simultaneously optimized. After convergence, the final architecture is derived by retaining only the module with the highest weight in each layer (i.e., $\arg\max_{{i \in \{1, \cdots, n_1\}}} \alpha_{l,i}$ for MPNN and $\arg\max_{{j \in \{n_1+1, \cdots, n\}}} \alpha_{l,j}$ for GT). The resulting subnets are subsequently retrained from scratch.

\section{Detailed Experimental Setup for Ranking Correlation}
\label{appendix_rank}

\noindent To conserve computational resources, we randomly reduce the Cluster dataset to 30\% of its original size and set the number of MPGT layers to eight, deviating from our standard setup. Subsequently, we randomly sample the number of 180 architectures from the whole search space, with each subnet trained for 100 epochs. We then compile these architectures and corresponding accuracy metrics into a rank dataset. In the ranking correlation experiment, each weight-sharing NAS method randomly searches for 50 subnets that satisfy the specified requirements (i.e., are among these 180 architectures). These 50 subnets adopt weights trained under the (sub-)supernet and are fine-tuned for a small number of epochs prior to performance evaluation. Finally, to compute the rank correlation coefficient, we rank these 50 architectures based on their performance recorded in the rank dataset and assess their estimated performance after fine-tuning. A higher coefficient suggests that the transferred weights more accurately reflect the true potential of the architectures.

\section{Detailed Comparison of GC and GM}
\label{appendix:gm}

\begin{figure}[!t]
\centering
\includegraphics[width=0.85\columnwidth]{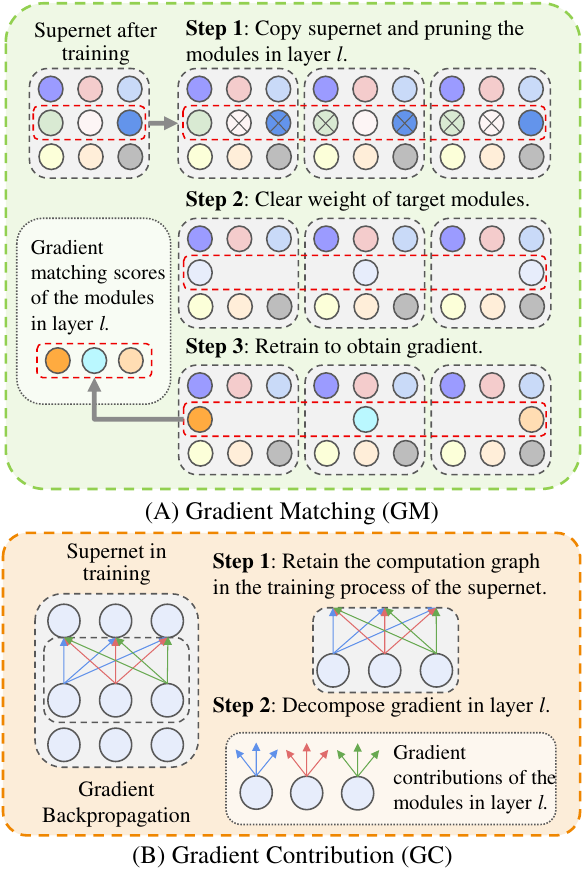} 
\caption{Illustration of the differences between GC and GM.}
\label{fig:comparsion_gm}
\vspace{-0.4cm}
\end{figure}

 \noindent Although both GM and the proposed GC are gradient-based supernet partition methods, their perspectives on analyzing the weight coupling problem are fundamentally different. Specifically, GC posits that the weight coupling problem primarily arises from modules in the $l$th layer delivering conflicting gradients to those in the $l-1$th layer. In contrast, GM computes gradients solely for modules within the same layer, without accounting for their impact on modules in the preceding layer. This fundamental difference in perspective leads to significant variations in the gradient computation methods and performance between these two supernet partition methods.

As shown in Figure~\ref{fig:comparsion_gm}, to compute the gradient for each module, GM duplicates the trained supernet to form a temporary supernet. In this temporary supernet, the gradient of the target module is first cleared while all other parallel modules in the same layer are pruned. This temporary supernet is then trained for a few training epochs, and the module's final weights after training are recorded as its gradient score. This method introduces two major problems:
(1) GM evaluates a specific module by pruning all others, which inherently alters the original supernet architecture. Consequently, it remains unclear whether the target module retains consistent functionality before and after pruning.
(2) GM requires pruning and retraining for each module across the entire supernet, which incurs substantial additional computational costs. To mitigate this, GM limits training for each temporary supernet to a single epoch. However, the gradients accumulated during such brief training are minimal, resulting in poor discriminability \cite{He2023}.

In contrast, GC only requires retaining the computation graph during supernet training. After training, GC directly decomposes this computation graph to extract the gradient contributions of different modules to those in the preceding layer without incurring additional training costs. 

\section{Rationale for Supernet Partitioning Based on Early-Stage Gradients}
\label{appendixf}
\begin{table}[!t]
\centering
\caption{Ranking correlation comparison with weight-sharing NAS methods on the Cluster dataset}
\label{tab:stage_cmp}
\scalebox{1} {\begin{tabular}{l|cc}
\hline
& \textbf{Kendall}\\
\hline
Early Stage vs. Middle Stage & $0.92$ \\
Early Stage vs. Late Stage & $0.86$ \\
Middle Stage vs. Late Stage & $0.98$ \\
\hline
\end{tabular}}
\vspace{-0.4cm}
\end{table}

\noindent In this section, we assess the ranking consistency of the cut weights across partitioning schemes derived from different stage module gradients. Specifically, we train a supernet and record the module gradients at the early, middle, and late stages. As shown in Table~\ref{tab:stage_cmp}, the experimental results demonstrate high ranking consistency across different stages, confirming that partition schemes derived from early‐stage conflicts remain representative throughout training. Moreover, using early-stage conflicts is more time-efficient. Therefore, we adopt the similarity of the gradient contributions between the modules in the early stages of the supernet as the partition criterion.

\begin{table}[!t]
\centering
\caption{Performance comparison with different numbers of MPGT layers on the Cluster dataset}
\label{tab:layer_config}
\begin{tabular}{c|ccccc}
\hline
 \textbf{Layer Number}  & \textbf{2} & \textbf{4} & \textbf{8} & \textbf{12} & \textbf{15} \\ 
\hline
\textbf{Acc. (\%)} & 50.97 & 69.72 & 76.89 & 77.70 & \textbf{78.14} \\ 
\hline
\end{tabular}
\vspace{-0.2cm}
\end{table}
\section{Rationale for Incorporating GC into UGAS}
 Unlike prior GNAS studies that restrict the search to only 2 or 3 layers of MPNNs, our UGAS jointly explores MPNNs and GTs, enabling deeper architectures and more diverse module combinations. Specifically, in UGAS, each layer has $2^n$ possible module configurations, resulting in a total search space of $2^{n\times L}$, where $n$ and $L$ denote the number of candidate modules and the number of layers. For example, when $n=6$ and $L=5$, the number of possible architectures exceeds $10^9$. This vast search space intensifies the weight-coupling challenge, making the integration of GC essential for UGAS. In addition, we conduct additional experiments to assess the performance of UGAS+GC with varying the number of layers. As shown in Table~\ref{tab:layer_config}, the experimental results demonstrate that increasing the number of layers can lead to improved performance.

\section{Performance Comparison of Different Minimum Cut Algorithms}
\label{appendixh}
\noindent  In this section, we assess the impact of different minimum cut algorithms on partition quality. As shown in Table~\ref{tab:partition_compare}, the experimental results indicate that the choice of algorithm has a negligible effect on the quality of the partition. Therefore, in this paper, we adopt the Stoer–Wagner algorithm in all experiments to divide modules into two sets.

\begin{table}[!t]
\centering
\caption{Performance comparison of different Minimum Cut Algorithms}
\label{tab:partition_compare}
\scalebox{0.9} {\begin{tabular}{c|ccccc}
\hline
\textbf{Algorithm} & \textbf{Stoer-Wagner} & \textbf{Karger} & \textbf{METIS} & \textbf{Spectral} & \textbf{Flow} \\
\hline
UGAS+GC & 0.661 & 0.659 & 0.657 & 0.660 & 0.661 \\
\hline
\end{tabular}}
\vspace{-0.2cm}
\end{table}

\section{Architecture Searched by UGAS+GC}
\begin{figure}[!t]
\centering
\includegraphics[width=0.85\columnwidth]{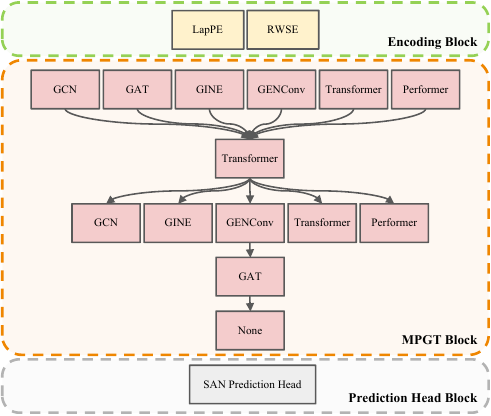} 
\caption{Illustration of the Searched Architecture on the Molsider Dataset.}
\label{fig:search_result}
\vspace{-0.4cm}
\end{figure}

\noindent In Figure~\ref{fig:search_result}, we present a well-performing architecture discovered on the Molsider Dataset. The MPGT Block in this architecture comprises five layers, where ``None" indicates that the layer does not include any GNN modules shown in Table~\ref{tab:search_space} and only contains an MLP layer. For the Encoding Block, following the prior study \cite{GraphGPS}, we directly utilize Laplacian Positional Encoding (LapPE) and Random Walk Structural Encoding (RWSE) to encode graph data, and a SAN-based prediction head~\cite{SAN} is adopted to perform predictions on the Molsider dataset.
\end{document}